\documentclass{article}

\usepackage{authblk}
\usepackage{lipsum}    
\usepackage[english]{babel}
\usepackage[hang, flushmargin]{footmisc}
\usepackage{indentfirst}
\usepackage{subcaption} 
\usepackage{titlesec}
\usepackage{multirow}
\usepackage{multicol}
\usepackage{array} 
\usepackage{tikz}
\usepackage{algorithm}
\usepackage{algorithmic}
\usepackage{listings}
\usepackage{amssymb}
\usepackage{booktabs}
\usepackage{amsmath}
\usepackage{amssymb}
\usepackage{array}
\usepackage{ragged2e}
\usepackage{xcolor}
\usepackage{enumitem}
\usepackage{float}
\usepackage{tabularx}
\usepackage{listings} 
\usepackage{makecell}
\usepackage{longtable}
\usepackage{lscape}
\usepackage{tokcycle} 
\usepackage{amsmath}
\usepackage{amsthm}
\usepackage{graphicx}
\usepackage[utf8]{inputenc}
\usepackage{csquotes}

\theoremstyle{definition}
\newtheorem{definition}{Definition}[section]

\newtheorem{proposition}{Proposition}[section]

\usepackage[colorlinks=true, allcolors=blue]{hyperref}

\usepackage[numbers,sort&compress]{natbib}
\usepackage[letterpaper,top=2cm,bottom=2cm,left=2.5cm,right=2.5cm,marginparwidth=1.75cm]{geometry} 

\usepackage[T1]{fontenc} 
\usepackage[sc]{mathpazo} 
\usepackage{microtype} 

\title{\textbf{GraphNet: A Large-Scale Computational Graph Dataset for Tensor Compiler Research}} 

\author{\normalsize
\textbf{Xinqi Li\thanks{Corresponding author: lixinqi04@baidu.com}\quad  Yiqun Liu\quad Shan Jiang\quad Enrong Zheng\quad Huaijin Zheng} \\ \textbf{Wenhao Dai\quad Haodong Deng\quad Dianhai Yu\quad Yanjun Ma}}
\date{\normalsize\vspace{-1em}PaddlePaddle Team, Baidu Inc.}
\usepackage{abstract}

\begin{document}

\maketitle 
\vspace{-1cm}
\section*{\centering Abstract}
\noindent We introduce GraphNet, a dataset of 2.7K real-world deep learning computational graphs with rich metadata, spanning six major task categories across multiple deep learning frameworks. To evaluate tensor compiler performance on these samples, we propose the benchmark metric Speedup Score $S_t$, which jointly considers runtime speedup and execution correctness under tunable tolerance levels, offering a reliable measure of general optimization capability. Furthermore, we extend $S_t$  to the Error-aware Speedup Score $ES_t$, which incorporates error information and helps compiler developers identify key performance bottlenecks. In this report, we benchmark the default tensor compilers, CINN for PaddlePaddle and TorchInductor for PyTorch, on computer vision (CV) and natural language processing (NLP) samples to demonstrate the practicality of GraphNet. The full construction pipeline with graph extraction and compiler evaluation tools is available at \href{https://github.com/PaddlePaddle/GraphNet}{https://github.com/PaddlePaddle/GraphNet}.

\section{Introduction}
\label{sec:intro}

The development of high-performance GPU kernels has become critical for computational efficiency in modern deep learning workloads. A widely adopted approach integrates deep learning frameworks, such as PaddlePaddle\cite{ma2019paddlepaddle} and PyTorch\cite{Pytorch}, with vendor-specific operator libraries (e.g., cuDNN~\cite{chetlur2014cudnn}, oneDNN~\cite{oneDNN_Contributors_oneAPI_Deep_Neural}). However, emerging hardware featuring low-precision formats (e.g., BF16\cite{wang2019bfloat16}, FP8\cite{micikevicius2022fp8formatsdeeplearning}) and advanced memory access patterns is creating a growing demand for custom operators-a demand that vendor-specific libraries often struggle to accommodate. To meet these demands, modern deep learning systems increasingly rely on tensor compilers (e.g., TVM~\cite{TVM}, XLA~\cite{XLA}, BladeDISC~\cite{BladeDISC}) to lower high-level computational graphs into efficient backend implementations for heterogeneous hardware platforms.

Despite these advances, deep learning engineers still rely heavily on manual tuning, as existing tensor compilation frameworks often lack fine-grained hardware control and offer limited support for cross-platform complex algorithms. Meanwhile, researchers are increasingly exploring the use of large language models (LLMs)~\cite{ouyang2025kernelbench} and AI coding agents~\cite{wang2025geakintroducingtritonkernel} to automatically generate efficient operators and kernels. In this challenging context, systematic evaluation of existing tensor compilers is essential to identify performance bottlenecks and guide the evolution of next-generation compilers.

Existing benchmarks are often ad hoc, relying on a small set of hand-picked samples, with limited coverage of up-to-date, real-world models from the community and lacking sufficient support for cross-framework evaluation. To address these limitations, we introduce \textbf{GraphNet}, a large-scale dataset of deep learning computational graphs designed to enable systematic evaluation of tensor compilers across tasks and frameworks. We further propose a unified metric to benchmark compiler performance by jointly accounting for runtime speedup, numerical correctness, and compilation failures. Our main contributions are as follows:

\begin{itemize}
    \item We collected over 2.7k computational graphs from real-world models, covering diverse task categories and mainstream frameworks (e.g., PaddlePaddle, PyTorch). All samples are stored in a unified format and are compatible with multiple tensor compiler backends.
    
    \item We evaluate the performance of CINN\cite{cinn2021} and TorchInductor on CV and NLP samples from GraphNet, and propose the \textbf{Speedup Score} ($S_t$) to measure the general optimization capability of tensor compilers, as well as the \textbf{Error-aware Speedup Score} ($ES_t$), which accounts for all types of execution failures.
    
    \item We present a detailed study of the dataset construction pipeline and sample constraints, and release the GraphNet dataset along with open-source extraction and evaluation tools.
\end{itemize}

The rest of this paper is organized as follows. Section~\ref{sec:properties} introduces dataset properties and its distribution. Section~\ref{sec:application} presents the design of our evaluation metric and discusses experimental results. Section~\ref{sec:cons-graphnet} describes the dataset construction methodology. Sections~\ref{sec:background} and~\ref{sec:conclusion}  review related works and outline future directions.
\section{Dataset Properties}
\label{sec:properties}
    \textbf{Authenticity:}  
    All samples are extracted from models in mainstream deep learning frameworks, focusing on PaddlePaddle libraries (e.g., PaddleNLP, PaddleX, PaddleScience) and PyTorch libraries (e.g., TorchVision, timm, mmseg). These libraries are actively maintained and widely adopted by the community, ensuring that GraphNet reflects real-world workloads rather than synthetic benchmarks.

     \textbf{Compatibility:}  
    GraphNet samples adopt a standardized format that integrates computational graphs with inputs, weights, and custom operators. This format preserves complete computation semantics without information loss, ensuring compatibility with diverse compiler backends, including CINN (PaddlePaddle), TorchInductor (PyTorch), XLA (TensorFlow), TVM, BladeDISC, and TensorRT. This enables fair and reproducible evaluation, as well as seamless extensions to new compilers.

    \textbf{Diversity:}  
The dataset spans six major task categories: computer vision, natural language processing, audio, multimodal learning, scientific computing, and others. Models vary in scale from a few thousand to 10B parameters, with the most complex sample containing up to 3,686 operators. GraphNet also supports multiple data types, such as BF16, FP16, and FP32, which are essential for evaluating compiler optimizations across different numerical precisions and hardware backends.

\begin{figure*}[htpb]
    \centering
    \captionsetup{font=small}
    \begin{subfigure}[c]{0.45\linewidth}  
        \centering
        \includegraphics[width=0.65\linewidth]{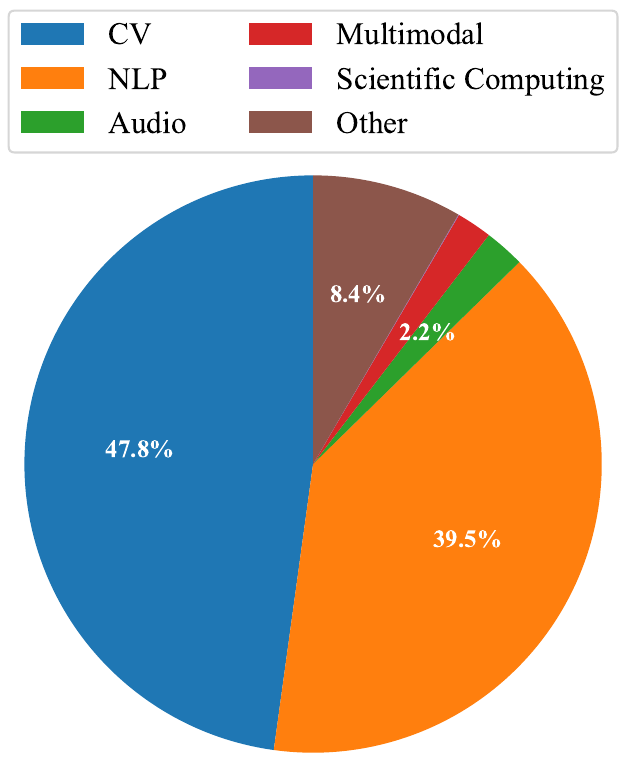}
        \caption{Categories of Computational Graph}
        \label{subfig:dataset_graph_type}
    \end{subfigure}
    \begin{subfigure}[c]{0.4\linewidth}  
        \centering
        \begin{subfigure}{\linewidth}
            \centering
            \includegraphics[width=\linewidth]{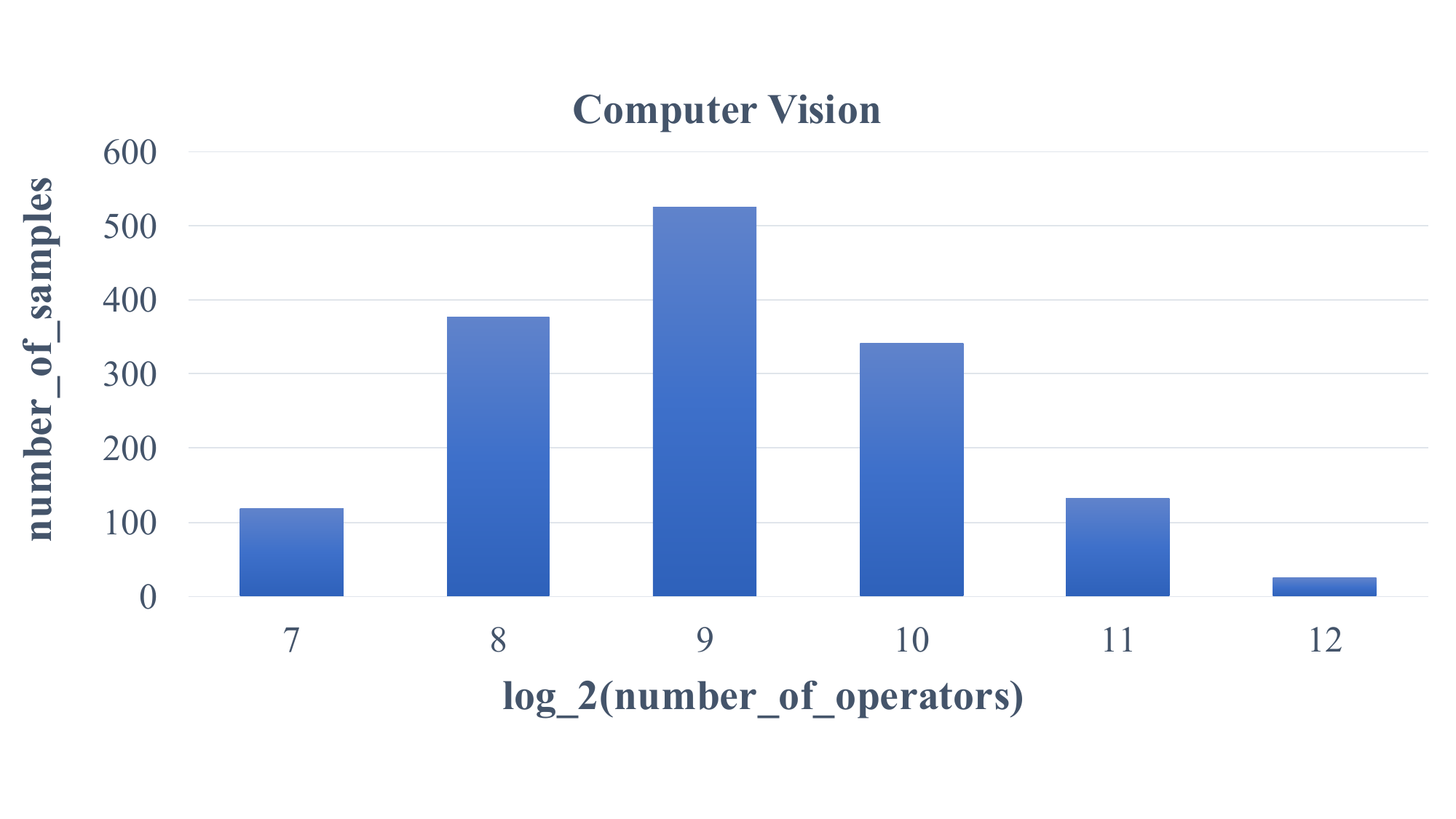}
            \caption{Operator Numbers of CV Models}
            \label{subfig:dataset_cv_op_nums}
        \end{subfigure}
        \begin{subfigure}{\linewidth}
            \centering
            \includegraphics[width=\linewidth]{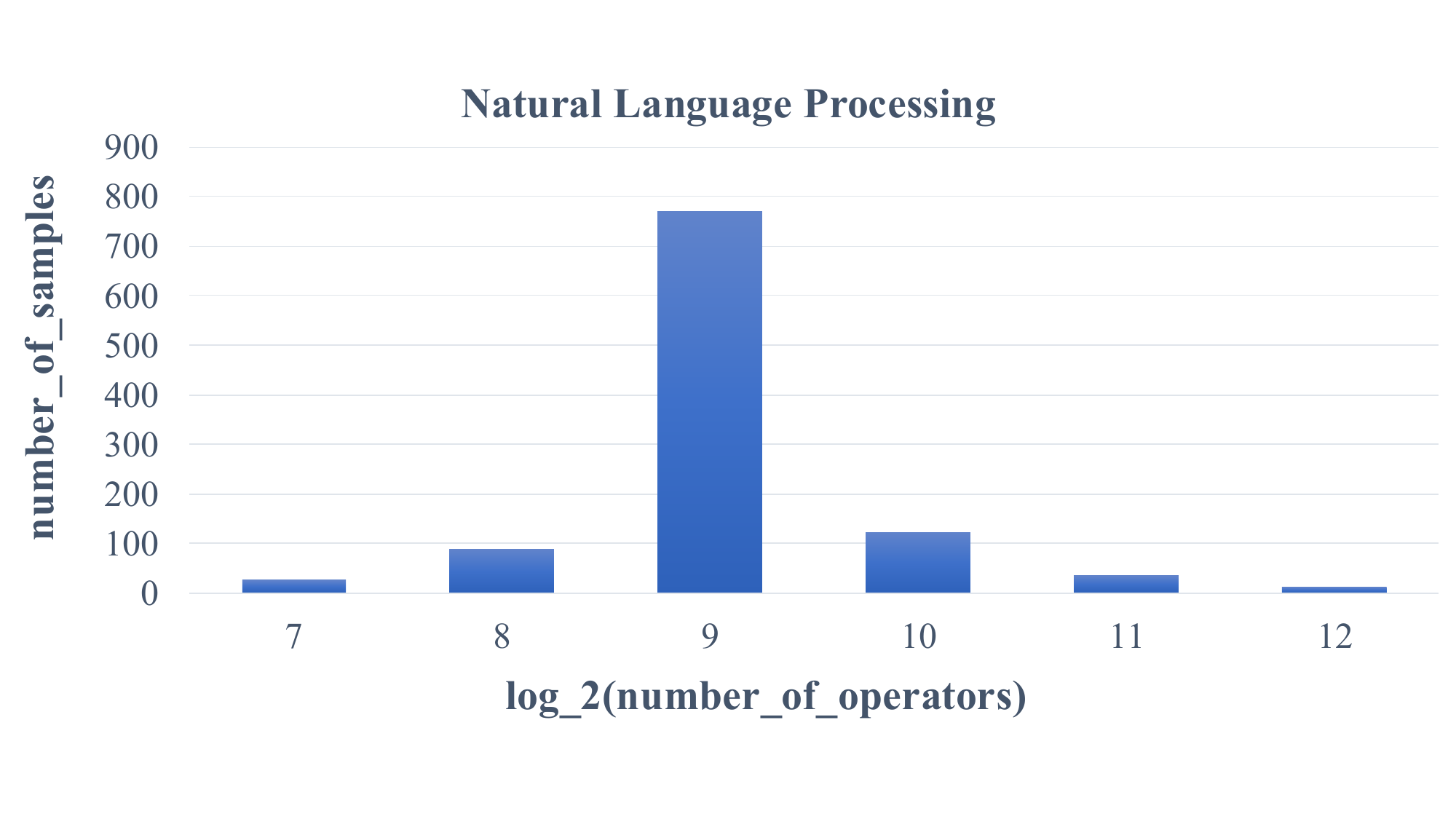}
            \caption{Operator Numbers of NLP Models}
            \label{subfig:dataset_nlp_op_nums}
        \end{subfigure}
    \end{subfigure}
    \hfill
    \caption{\textbf{Statistical properties of the GraphNet dataset. }
(a) The distribution of computational graphs across six major task categories, showing that Computer Vision (47.8\%) and Natural Language Processing (39.5\%) are the dominant domains. 
(b) and (c) Histograms showing the distribution of operator counts (on a log$_2$ scale) for CV and NLP models, respectively. Both categories show a high concentration of graphs with operator counts around $2^9$ (512).}
    \label{fig:properties_computational_graph}
\end{figure*}

\section{Benchmark of Tensor Compilers}
\label{sec:application}
\subsection{Compiler Evaluation}
\label{subsec:compiler-evaluation}

We introduce an automated compiler evaluation workflow that enables unified and repeatable testing of multiple tensor compilers. The evaluation takes GraphNet samples as input and consists of the following key steps:

\begin{enumerate}
    \item \textbf{Baseline Execution:} Execute the original model in Eager mode on a given framework and record its output and execution time $T_{eager}$ as the baseline for subsequent comparisons.
    
    \item \textbf{Compiler Configuration:} Compile the original model into an optimized executable by specifying the target compiler. Currently supported compilers include CINN, TorchInductor, XLA, TVM, BladeDISC, and TensorRT.
    
    \item \textbf{Compiled Execution:} Execute the compiled model after a warmup phase to eliminate cold-start overhead and obtain the pure execution time $T_{compiled}$. Record its output to verify accuracy.

    \item \textbf{Performance Analysis:} Compare the outputs collected from baseline and compiled execution to validate correctness. Quantify compiler performance by considering speedup ratio $ T_{eager}/T_{compiled}$, execution correctness, and performance degradation. The details of the evaluation metrics are discussed in Section~\ref{subsec:eva-metrics}.
\end{enumerate}

\subsection{Evaluation Metrics}
\label{subsec:eva-metrics}

We design two distinct metrics: one for benchmarking tensor compilers and one for compiler development. The benchmarking metric, detailed in Section~\ref{subsubsec:metrics-benchmark}, focuses primarily on acceleration and treats the error penalty as a fixed constant. The development metric (Section~\ref{subsubsec:metrics-development}), exposes detailed error penalty  information, allowing developers to configure it on demand.

\subsubsection{Metrics for Compiler Benchmark}
\label{subsubsec:metrics-benchmark}

The performance evaluation metric is parameterized by a tunable \textbf{tolerance} $t$,  which enables the metric to capture compiler performance under varying correctness criteria, from strict numerical precision to more relaxed settings. Tolerance $t$ determines whether a sample is considered correctly executed, based on the check  
\texttt{assert\_close(x, y, rtol(t), atol(t))}.  
The values of relative tolerance (rtol) and absolute tolerance (atol) corresponding to different $t$ settings and data types are summarized in Appendix~\ref{sec:appendix_atol2rtol}.

Based on tolerance $t$ and its associated correctness criteria, we define the \textbf{Speedup Score} $S_t$ as a unified metric that combines speedup, accuracy, and penalty factors, providing a comprehensive measure of compiler performance:
\begin{equation}
S_t = \alpha^{\lambda} \cdot \beta^{\lambda \eta p} \cdot b^{1-\lambda}
\label{eq:St-macro-form}
\end{equation}

\noindent For readability, we omit the explicit $(t)$ notation in this expression, while noting that $\alpha, \beta, \lambda$, and $\eta$ are all dependent on the tolerance $t$. Eq.~\eqref{eq:St-macro-form} can be viewed as the product of three components:

\begin{itemize}
  \item \textbf{Correct executions:} 
  $\alpha^\lambda$ where $\alpha$ is the geometric mean speedup of all correctly executed samples, 
  and $\lambda$ is the fraction of correctly executed samples, weighting the contribution of this component.

  \item \textbf{Performance degradation:} 
  $\beta^{\lambda \eta p}$ where $\eta$ is the fraction of correctly executed samples that have speedup $<1$ (i.e., slowdowns among correct executions), $\beta$ is the geometric mean speedup of these slowdown cases, 
  and $p\!\in\!(0,1)$ is the penalty (default 0.1) applied to emphasize their impact.

\item \textbf{Failures:} 
  $b^{1-\lambda}$ where $1-\lambda$ denotes the fraction of samples that execute incorrectly 
  (tolerance violations, compilation failures, or runtime crashes), and $b \in (0,1)$ is the penalty factor (default 0.1).
\end{itemize}

The equivalence between the macro formulation and its sample-level interpretation is formally proven in Proposition~\ref{prop:equivalence-St-GMRS} (Appendix~\ref{sec:appendix-analysis-st}).

\textbf{Experiment Result:}  
We conducted experiments on the NVIDIA H20 GPU using CV and NLP samples from GraphNet, evaluating \textbf{CINN} (the default compiler backend of PaddlePaddle) and \textbf{TorchInductor} (the default compiler backend of PyTorch) on their respective frameworks. As the workloads originate from their native ecosystems, we refer to the two settings as \textbf{PaddlePaddle} and \textbf{PyTorch} in the following discussion. Detailed experimental settings, including framework versions, are summarized in Appendix~\ref{sec:benchmark-details}.

\begin{figure}[htpb]
    \centering
    \captionsetup{font=small}
    \includegraphics[width=0.75\linewidth]{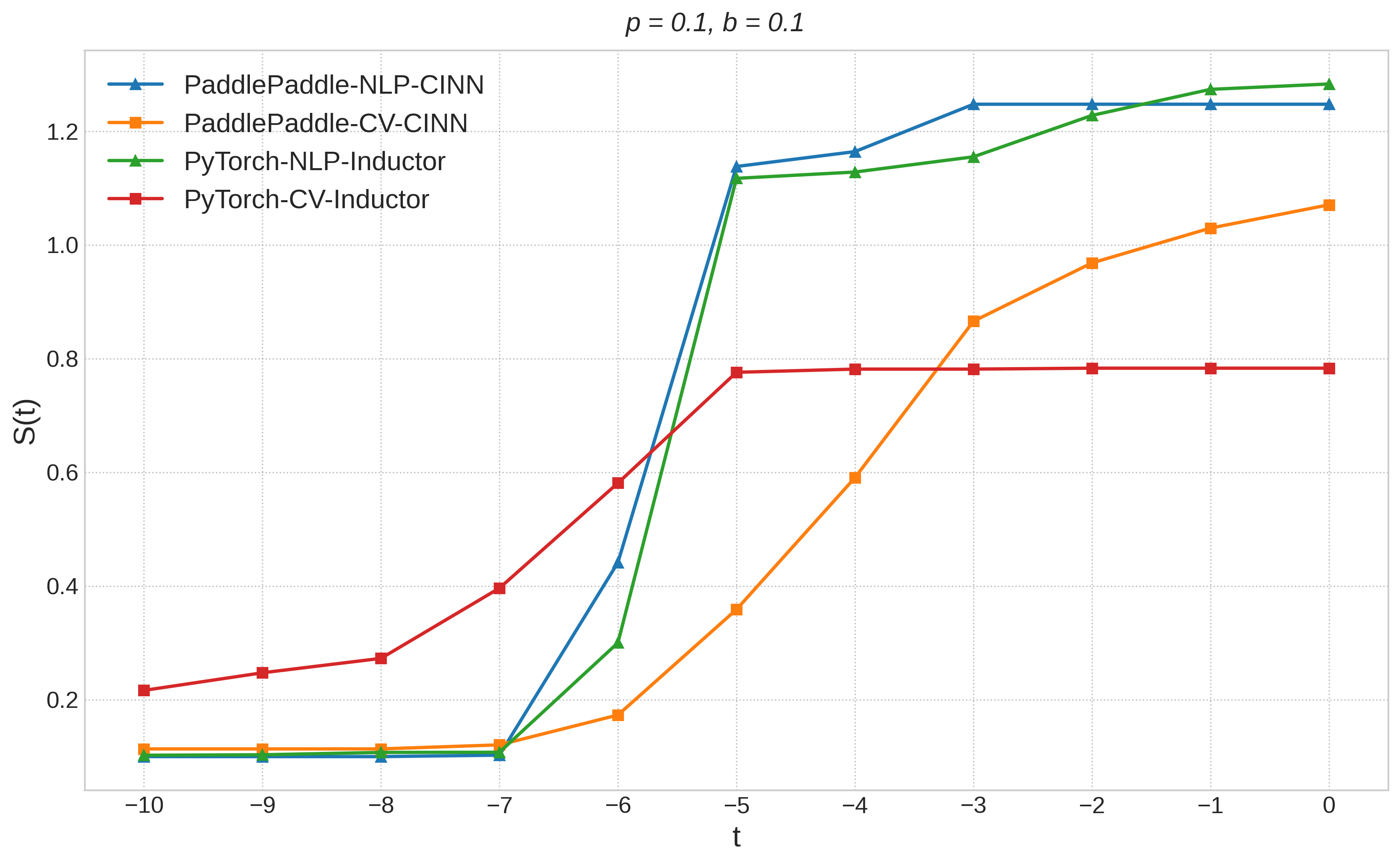}
    \caption{\textbf{Speedup Score $S_t$ on NVIDIA H20 for CV and NLP workloads.} The vertical axis shows $S_t$, which integrates speedup, pass rate, and failure penalties into a unified score. Higher $S_t$ indicates better compiler performance under the correctness-aware speedup metric defined in Equation~\ref{eq:St-macro-form}.
The horizontal axis $t$ represents different numerical tolerance levels used for correctness checks, where larger $t$ implies more relaxed thresholds.}
    \label{fig:experiment-St}
\end{figure}

As shown in Figure~\ref{fig:experiment-St}, the $S_t$ trajectories exhibit distinct trends across workloads and compiler configurations. For both PaddlePaddle-NLP (blue) and PyTorch-NLP (green), $S_t$ remains low under strict tolerances ($t \leq -7$) but rises sharply within $t \in [-6, -5]$, indicating that most samples pass the correctness check once the tolerance is slightly relaxed. Both compilers converge to similar peak values around $S_t \approx 1.2$ at $t = 0$, with PaddlePaddle achieving a marginally higher score.

In contrast, PaddlePaddle-CV (orange) exhibits a smoother performance increase from $t=-7$ to $t=0$, whereas PyTorch-CV (red) reaches a peak $S_t \approx 0.8$ at $t=-5$. and then plateaus. This flat trajectory suggests that increasing the tolerance further (beyond $t=-5$) no longer improves the score by resolving accuracy issues, which implies that its accuracy-related failures are effectively resolved at this point. However, its performance score remains constrained by an upper bound well below 1.0, indicating heavy penalties from either widespread execution failures or performance degradation (slowdowns). While the $S_t$ score reflects this combined impact, it does not separate these root causes. This breakdown is provided by the development metrics in the following section.

\subsubsection{Metrics for Compiler Development}
\label{subsubsec:metrics-development}

For compiler developers, error information from incorrectly executed samples is crucial. To encode this information into the evaluation metric, we reinterpret the positive tolerance domain $t\in(0,+\infty)$ to represent discrete levels of error tolerance. We assign error codes $c \in \{1, 2, 3\}$ to accuracy errors, runtime crashes, and compilation failures, respectively. The metric then tolerates errors based on the level $t$: $t \ge 1$ tolerates accuracy errors ($c=1$), $t \ge 2$ additionally tolerates runtime crashes ($c=2$), and $t \ge 3$ tolerates all failures ($c=3$).

Based on these discrete tolerance levels and error codes,
we extend the fixed penalty factor $b$ in $S_t$ (Eq.~\ref{eq:St-macro-form})
to a tolerance-dependent form $\gamma_t$:
\begin{equation}
\gamma_t = \prod_c (b^{\mathbb{1}(t<c)})^{\pi_c}
          = b^{\sum_c \pi_c \, \mathbb{1}(t<c)}
\label{eq:gamma_def}
\end{equation}
where $c\!\in\!\{1,2,3\}$ is the error code,
$\pi_c$ denotes the proportion of error code $c$ among all erroneous samples,
and $\mathbb{1}(\cdot)$ is an indicator function.
As $t$ increases, more error types are tolerated and $\gamma_t$ monotonically increases from $b$ to $1$.
With $\gamma_t$ defined, we obtain the \textbf{Error-aware Speedup Score}:
\begin{equation}
ES_t = \alpha^{\lambda} \cdot \beta^{\lambda \eta p} \cdot \gamma_t^{1-\lambda}
\label{eq:ESt-macro-form}
\end{equation}
For $t\le0$, $ES_t$ reduces to the original $S_t$.
As $t$ grows over the positive domain,
more categories of errors are tolerated and $ES_t$ increases monotonically.
When all errors are tolerated ($t\ge3$), $ES_t$ reaches its maximum,
representing the theoretical upper bound of achievable compiler performance.

Appendix~\ref{sec:appendix-analysis-ESt} analyzes the sample-level meaning of $ES_t$.

\textbf{Experiment Result:} We obtain $ES_t$ under the same conditions as in Section~\ref{subsubsec:metrics-benchmark}, as shown in Figure~\ref{fig:experiment-ESt}.

\begin{figure}[htpb]
    \centering
    \captionsetup{font=small}
    \includegraphics[width=0.75\linewidth]{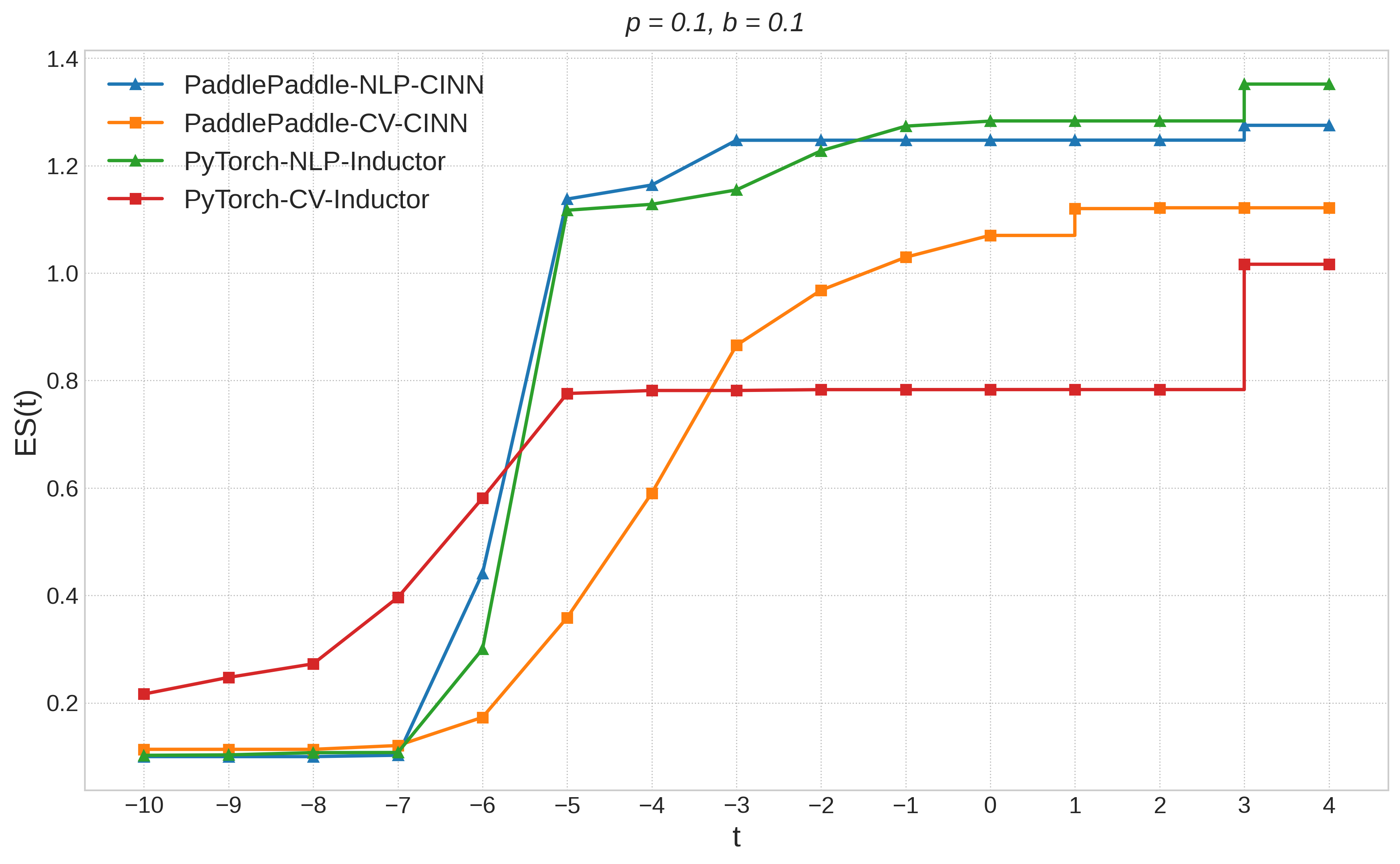}
    \caption{
\textbf{Error-aware Speedup Score $ES_t$ for CV and NLP workloads.} The vertical axis shows $ES_t$ values, capturing compiler performance with increasing fault tolerance.
Higher $ES_t$ indicates better compiler performance under the error-tolerant speedup metric defined in Equation~\ref{eq:ESt-macro-form}.
The horizontal axis $t$ represents different tolerance levels:
for $t \le 0$, it reflects numerical correctness thresholds (as in $S_t$);
for $t > 0$, it encodes the categories of tolerated errors:
$t \geq 1$ tolerates accuracy mismatches,
$t \geq 2$ tolerates runtime crashes,
and $t \geq 3$ tolerates compilation failures.
}
    \label{fig:experiment-ESt}
\end{figure}

In Figure~\ref{fig:experiment-ESt}, PaddlePaddle-CV (orange) exhibits a discrete jump from $t=0$ to $t=1$, indicating that a portion of its failures are due to substantial accuracy violations that are tolerated once $t \ge 1$. Jumps also occur at $t=3$ for both PyTorch-NLP (green) and PaddlePaddle-NLP (blue), revealing that a fraction of samples fail due to compilation errors. Notably, PyTorch-CV (red) shows a sharp rise at $t=3$, suggesting that a large number of its samples do not pass the compilation stage.

Finally, when $t \geq 3$, $ES_t$ no longer incorporates any penalties for accuracy violations, runtime crashes, or compilation errors, and thus reflects the raw speedup score the theoretical upper bound of compiler performance. PyTorch-CV ultimately plateaus around $ES_t \approx 1.0$, indicating that even with all failures ignored, its average performance is neutral. This suggests many of its samples are negatively optimized (i.e., compilation leads to slower execution). As shown in Figure~\ref{fig:experiment-violinplot} (Appendix~\ref{sec:benchmark-details}), the distribution of per-sample speedups confirms that PyTorch-CV contains a substantial number of negatively optimized samples.

While $S_t$ and $ES_t$ provide an aggregated view of compiler performance, Tables~\ref{tab:paddle-nlp-values}--\ref{tab:pytorch-cv-values} in the appendix provide detailed values of each component ($\alpha, \beta, \lambda, \eta, \gamma$), offering a fine-grained view of how $S_t$ and $ES_t$ are constructed.

\section{Construction of GraphNet}
\label{sec:cons-graphnet}
GraphNet provides a unified workflow for automated graph extraction, validation, and cross-backend performance evaluation. This section details the dataset requirements and construction methodology.
\subsection{Dataset Constraints}
\label{subsec:constraints}
We define five constraints applied to every computational graph in GraphNet to ensure overall dataset quality and cross-platform compatibility. These constraints act as validation requirements during dataset construction, ensuring consistency and usability at the source. Specifically, all computational graphs must satisfy the following constraints:
\begin{itemize}
    \item \textbf{Runnable}:  Each computational graph must successfully execute forward propagation under the designated framework without syntax errors, type mismatches, or runtime crashes.
    \item \textbf{Serializable}: Each sample and its associated metadata (e.g., input shapes, weight parameters) must be serializable into standard formats (e.g., JSON) and correctly de-serializable upon reloading.
    
    \item \textbf{Decomposable}: The entire computational graph must be decomposable into multiple non-overlapping subgraphs, where each subgraph represents an independent optimization unit. This supports compiler backends in performing fusion, scheduling, and other optimization tasks.
    
    \item \textbf{Statically Analyzable}: The name, type, attributes, and dependency relations of all operators must be statically extractable (e.g., via torch.fx) without model execution. This allows automated analysis tools to fully interpret operator semantics for structural traversal and pattern matching.
    
    \item \textbf{Custom Operator Accessible}: If a sample includes user-defined custom operators, the corresponding source code for these operators must be traceable and accessible in a modular form, ensuring reusability and integration across compiler environments. 
    \end{itemize}
\textbf{Remark:} At the time of writing, the first three constraints (Runnable, Serializable, and Statically Analyzable) are strictly enforced by our current validation pipeline. The remaining two (Decomposable and Custom Operator Accessible) are actively under development and represent our roadmap for expanding the dataset’s capabilities.
\subsection{Construction Methodology}
\begin{figure}[htpb]
    \centering
    \captionsetup{font=small}
    \includegraphics[width=0.75\linewidth]{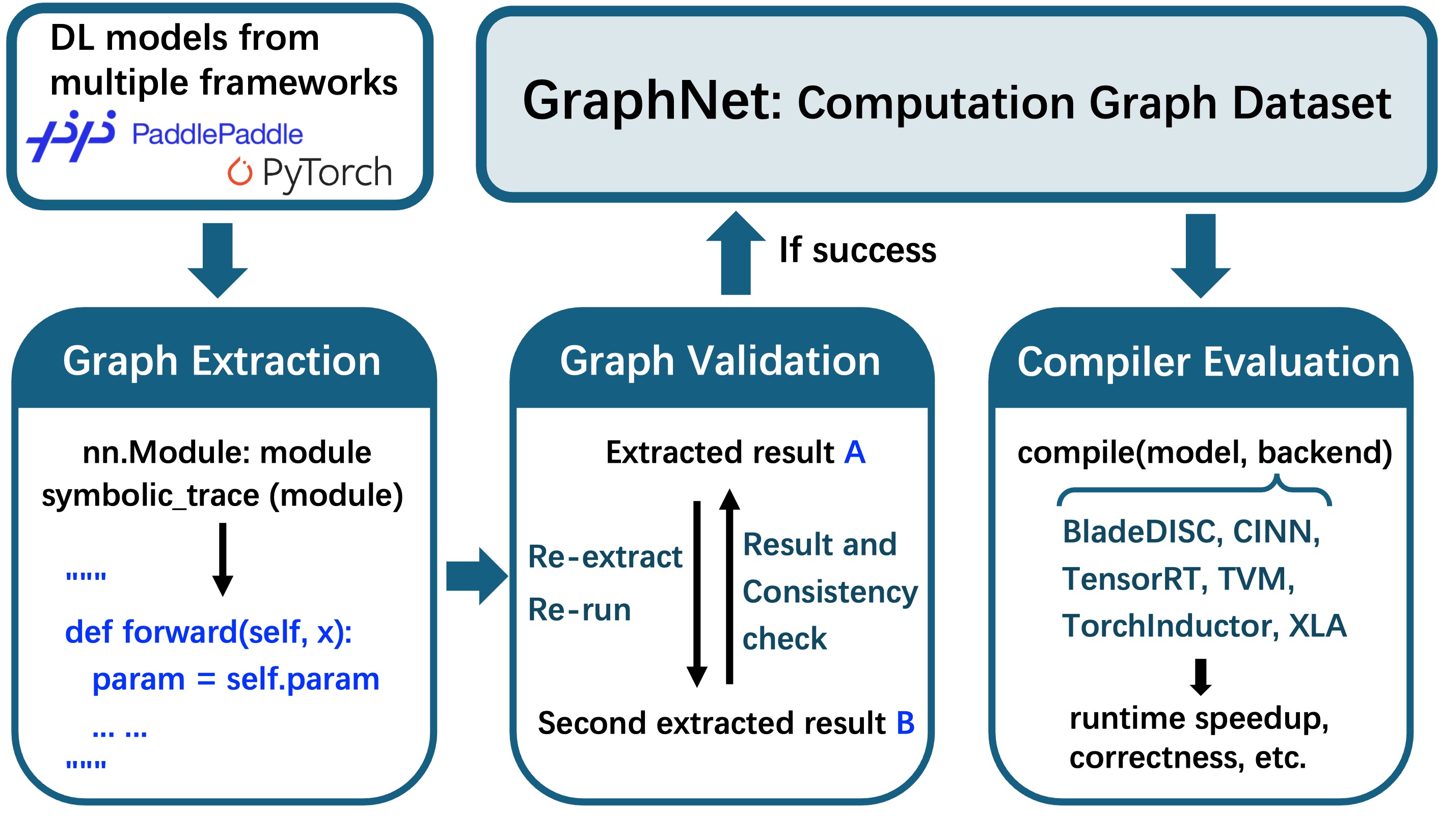}
    \caption{\textbf{GraphNet Workflow Overview.} 
    The workflow consists of three stages: 
    (1) \textbf{Graph Extraction}: Traces and captures computational graphs from DL models (e.g., PaddlePaddle, PyTorch); 
    (2) \textbf{Graph Validation}: Performs a consistency check via re-extraction and re-execution to ensure usability; 
    and (3) \textbf{Compiler Evaluation}: Uses the validated graphs from the dataset to benchmark the runtime speedup and correctness of various compiler backends.}
    \label{fig:GraphNet_workflow}
\end{figure}
GraphNet provides an automated workflow for collecting computational graphs across multiple platforms  while enforcing the constraints introduced in Section~\ref{subsec:constraints}. The dataset construction pipeline, illustrated in Figure~\ref{fig:GraphNet_workflow}, consists of two primary components: graph extraction and graph validation.

\subsubsection{Graph Extraction}
We design a lightweight extraction mechanism that captures dynamic computational graphs from real models and saves them as standardized GraphNet samples. As illustrated in the "Graph Extraction" component of Figure~\ref{fig:GraphNet_workflow}, we first create Python decorator-based interfaces, including \texttt{graph\_net.paddle.extract} and \texttt{graph\_net.torch.extract}. Users can wrap the target model with the extractor, which automatically triggers the graph extraction process at runtime. During execution, the extractor employs symbolic tracing and dynamic graph tracking mechanisms built into the framework to capture all operator invocations and tensor dependencies, thereby generating a complete dynamic computational graph. The captured graph is stored as a standardized set of files (as shown in Figure~\ref{fig:GraphNet_composition}), which encompasses the high-level IR of the computational graph in \texttt{model.py}, weights and input metadata, and optional custom operator implementations. Together, these components form a complete GraphNet sample.

\subsubsection{Graph Validation}
To ensure that the extracted data meets the dataset constraints, we adopt a re-extraction and re-execution mechanism during the validation stage. This workflow, illustrated in the center of Figure~\ref{fig:GraphNet_workflow}, consists of four steps:
First, the validator takes an extracted sample~A (the original computational graph) and deserializes it into an executable Python function, reconstructing the model structure with its input and weight metadata.
Second, the reconstructed model is executed to verify that the computational graph is runnable.
Third, the validator performs the extraction process again on the reconstructed model, producing a second computational graph~B.
Finally, the validator repeats the reconstruction and execution on graph~B, comparing its outputs with those from the second stage and checking for structural and node-level consistency between graphs~A and~B.

This validation procedure inherently enforces all dataset constraints. Failures in serialization or deserialization terminate the reconstruction process; non-runnable models are detected during execution; and inaccessible custom operator code interrupts re-extraction. In addition, the consistency check guarantees that static analysis of all operators is performed completely and correctly.

GraphNet also incorporates a graph deduplication mechanism during data collection to eliminate redundancy. For each extracted computational graph, a unique \textit{graph hash} value is generated from the model’s source code and graph topology. The validator then identifies and removes duplicate samples by comparing their hash values, ensuring that only distinct computational graphs are retained in the final dataset.
\begin{figure}[htpb]
    \centering
    \captionsetup{font=small}
    \includegraphics[width=0.75\linewidth]{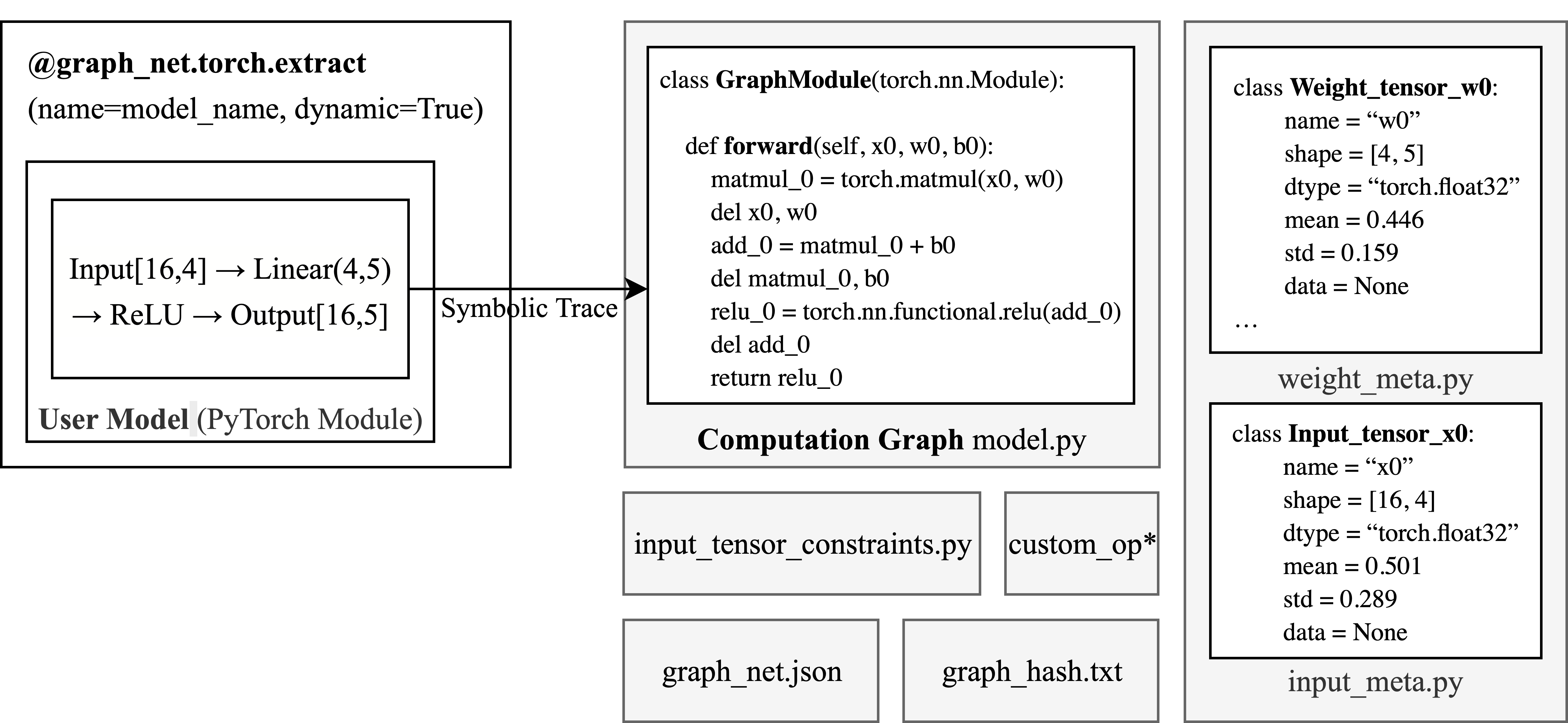}
    \caption{\textbf{GraphNet Sample Composition.}
    A user's model (left), wrapped by the \texttt{@graph\_net} extractor, is symbolically traced to generate a standardized set of files. 
    This set forms a complete sample, including the high-level IR of the computation graph (\texttt{model.py}), metadata for inputs and weights (\texttt{input\_meta.py}, \texttt{weight\_meta.py}), and other components such as optional custom operator code.}
    \label{fig:GraphNet_composition}
\end{figure}

\section{Background and Related Works}
\label{sec:background}
\textbf{Tensor Compiler:}  
Tensor compilers\cite{Li_2021} transform high-level computation graphs into optimized device-specific kernels through IR lowering and scheduling. Systems such as TVM ~\cite{TVM} and Ansor~\cite{Ansor} represent search-based compilers that rely on auto-tuning with cost models, while XLA~\cite{XLA} and Glow~\cite{glowgraphloweringcompiler} follow a heuristic-based approach with graph-level optimizations. Recent works include MetaSchedule~\cite{MetaSchedule}, Hidet~\cite{hidet}, and BladeDISC~\cite{BladeDISC}, which improve scheduling abstraction, tuning efficiency, and dynamic-shape support. Industry systems include framework-specific compilers such as CINN (PaddlePaddle) and TorchInductor (PyTorch), which are tightly integrated into their respective DL frameworks, as well as vendor-specific tools like TensorRT, which are specialized for the NVIDIA ecosystem.

\textbf{Performance Benchmarks:}   
DL model performance evaluation can be traced back to DeepBench~\cite{deepbench}, which measured the performance of fundamental operations such as matrix multiplication across different hardware platforms. This direction was later extended by the industry-wide MLPerf~\cite{mattson2020mlperf} suite, standardizing end-to-end evaluation of training and inference workloads. More recently, CompilerGym~\cite{compilergym} introduced benchmark datasets and reinforcement learning environments aimed at leveraging ML to improve compiler optimization, while KernelBench~\cite{ouyang2025kernelbench} proposed a benchmark for evaluating correctness and speedup of LLM-generated GPU kernels. In parallel, Furutanpey et al.\cite{leveragingneuralgraphcompilers} evaluated graph compilers across heterogeneous settings and introduced NGraphBench to address the gap between theoretical performance and actual deployment, emphasizing the importance of incorporating compiler effects into ML research.

\section{Conclusion and Future Work}
\label{sec:conclusion}

In this paper, we introduced GraphNet along with $S_t $ and $ ES_t$ metrics to enable reproducible evaluation of tensor compilers across tasks and frameworks. GraphNet provides an open and well-structured resource for kernel research at the computational graph level, ensuring broad coverage of real-world workloads. Through experiments, we showed that these tools allow researchers to gain an objective and comprehensive perspective on compiler optimization capability and to better identify potential performance bottlenecks.

Our future roadmap focuses on expanding GraphNet to better serve compiler researchers and developers:

\subsection{Completing GraphNet}
\textbf{Framework Expansion}  
GraphNet currently supports only two deep learning frameworks, and its evaluation is limited to NVIDIA GPUs (e.g., H20, A100). We plan to expand the dataset by adding more samples from additional frameworks such as TensorFlow, JAX, and MindSpore, and extend evaluation to a wider range of hardware platforms, including TPUs and NPUs. 

\textbf{Task Category Refinement}  
We will introduce more fine-grained categories within the existing six major domains (CV, NLP, Audio, Multimodal, Scientific Computing, and Other). This will enable more targeted evaluation of compilers on specific application scenarios. 

\textbf{Sample Feature Enhancement}  
We aim to enrich sample features by enhancing our graph decomposition tools, allowing full graphs to be more easily split into disjoint subgraphs. We also seek to broaden support for more complex custom operators to preserve model-specific functionality.

\textbf{Distributed Scenario Support}
Finally, we will extend GraphNet to incorporate distributed computing scenarios, so that GraphNet can capture computation graphs with communication operators and support the evaluation of compiler optimizations in large-scale distributed systems.


\subsection{Applying GraphNet}

\textbf{Systematic Compiler Evaluation}  
Beyond the limited experiment on CINN and TorchInductor, GraphNet can be extended to benchmark a broader range of tensor compilers under a unified metric. From the user’s perspective, this enables users to select compilers based on task categories and framework requirements. From the compiler developer’s perspective, it helps developers quickly identify worst-case scenarios and uncover optimization bottlenecks.

\textbf{High-level IR Translation}  GraphNet includes computation graphs from multiple deep learning frameworks, providing a foundation for high-level IR translation. Such translation unifies graphs across ecosystems, ensuring that compiler evaluations are based on fully aligned datasets.

\textbf{AI for Compiler Research}
GraphNet can serve as training and evaluation data for AI-generated compiler passes and kernels. These AI-generated optimizations can be directly applied to GraphNet samples and systematically compared with existing compiler backends under the $ES_t$ metric. This enables a fair measurement of their speedup and correctness, and positions GraphNet as a benchmark platform for advancing AI-driven compiler research.


\section{Acknowledgment}
\label{sec:acknowledgment}
We thank the developers from the PaddlePaddle community for their invaluable support in constructing GraphNet. 
In particular, we especially acknowledge Zichao Xia, Ruqi Yang for their key contributions, and we also gratefully recognize the efforts of Guoyong Fang, Mengyuan Liu, Min Li, Shun Liu, Xin Wang, Xujun Chen, Yimeng Xu, Yiqiao Zhang, and Zeping Wu.


\appendix

\bibliography{ref}


\newpage

\section*{\centering Appendix}

\section{Configuration of atol(t) and rtol(t)}
\label{sec:appendix_atol2rtol} 


\begin{table}[h]
\centering
\renewcommand{\arraystretch}{1.5}

\begin{minipage}[t]{0.48\textwidth}
\centering
\begin{tabular}{|c||c|c|c|}
\hline
\textbf{Data Type} & \textbf{atol(t)} & \textbf{atol}(-5) & \textbf{atol}(0)  \\
\hline \hline
float16    &  $10^{t}$   & 1e-5   &  $1$    \\ \hline
bfloat16   &  $10^{t}$   & 1e-5   &  $1$    \\ \hline
float32    &  $10^{t}$   & 1e-5   &  $1$    \\ \hline
float64    &  $10^{t\cdot7/5}$ & 1e-7   &  $1$    \\ \hline
complex32  &  $10^{t}$   & 1e-5   &  $1$    \\ \hline
complex64  &  $10^{t}$   & 1e-5   &  $1$    \\ \hline
complex128 &  $10^{t\cdot7/5}$   & 1e-7   &  $1$   \\ \hline 
quint8     &  $10^{t}$   & 1e-5   &  $1$    \\ \hline
quint2x4   &  $10^{t}$   & 1e-5   &  $1$    \\ \hline
quint4x2   &  $10^{t}$   & 1e-5   &  $1$    \\ \hline
qint8      &  $10^{t}$   & 1e-5   &  $1$    \\ \hline
qint32     &  $10^{t}$   & 1e-5   &  $1$    \\ \hline
others     &  $0.0$   & $0.0$         &  $0.0$  \\
\hline
\end{tabular}
\caption{atol configuration}
\end{minipage}
\hfill
\begin{minipage}[t]{0.48\textwidth}
\centering
\begin{tabular}{|c||c|c|c|}
\hline
\textbf{Data Type} & \textbf{rtol(t)} & \textbf{rtol}(-5) & \textbf{rtol}(0)  \\
\hline \hline
float16     & $10^{t\cdot3/5}$         & 1e-3       & $1$    \\ \hline
bfloat16    & $10^{t\cdot1.796/5}$         & 1.6e-2   & $1$  \\ \hline  
float32     & $10^{t\cdot5.886/5}$         & 1.3e-6   & $1$  \\ \hline  
float64     & $10^{t\cdot7/5}$         & 1e-7       & $1$    \\ \hline
complex32   & $10^{t\cdot3/5}$         & 1e-3       & $1$    \\ \hline
complex64   & $10^{t\cdot5.886/5}$         & 1.3e-6   & $1$  \\ \hline  
complex128  & $10^{t\cdot7/5}$         & 1e-7       & $1$    \\ \hline
quint8      & $10^{t\cdot5.886/5}$         & 1.3e-6   & $1$  \\ \hline  
quint2x4    & $10^{t\cdot5.886/5}$         & 1.3e-6   & $1$  \\ \hline  
quint4x2    & $10^{t\cdot5.886/5}$         & 1.3e-6   & $1$  \\ \hline  
qint8       & $10^{t\cdot5.886/5}$         & 1.3e-6   & $1$  \\ \hline  
qint32      & $10^{t\cdot5.886/5}$         & 1.3e-6   & $1$  \\ \hline  
others      & $0.0$       & $0.0$           & $0.0$  \\
\hline
\end{tabular}
\caption{rtol configuration}
\end{minipage}
\end{table}

Building on PyTorch's default testing settings, we develop a log-linear interpolation scheme to construct a (atol, rtol) configuration table. Specifically, we perform log-linear interpolation between two reference points, for instance, $\mathrm{atol}_{\mathrm{fp32}}(-5) = 10^{-5}$ and $\mathrm{atol}_{\mathrm{fp32}}(0) = 1$. In essence, this means that $\lg(\mathrm{atol}(t))$ and $\lg(\mathrm{rtol}(t))$ vary linearly with $t$, leading to the unified representation $\mathrm{atol}(t),\mathrm{rtol}(t) = 10^{kt}$. This formulation provides a coherent framework that encompasses floating-point (float16, bfloat16, float32, float64), complex, and quantized integer types. As a result, the tolerance $t \in (-\infty, 0]$ is mapped smoothly onto tolerance bounds across diverse precisions.

\newpage
\section{Sample-level Interpretation of \texorpdfstring{$S_t$}{St}}
\label{sec:appendix-analysis-st}

In this section, we show that the macro-level Speedup Score $S_t$ (Eq.~\ref{eq:St-macro-form}) can be equivalently expressed as the geometric mean of per-sample rectified speedups.

\begin{definition}[Rectified Speedup]
\label{def:rectified-speedup}
For each test sample $i$, the \emph{rectified speedup} under tolerance $t$ is defined as:
\begin{equation}
\tilde{s}_{t,i} =
\begin{cases}
  s_i, & \text{if } \mathrm{correct}_{t,i} \land s_i \ge 1, \\[3pt]
  s_i^{\,p+1}, & \text{if } \mathrm{correct}_{t,i} \land s_i < 1, \\[3pt]
  b, & \text{if } \neg \mathrm{correct}_{t,i},
\end{cases}
\label{eq:rectified_speedup}
\end{equation}
where \(s_i\) denotes the raw speedup ratio, \(\mathrm{correct}_{t,i}\) indicates whether the execution satisfies the tolerance criterion \(t\), \(p \in (0,1)\) is the degradation penalty coefficient, and \(b \in (0,1)\) is the failure penalty.
This formulation ensures that:
(i) correct executions with speedup retain their measured gain;
(ii) correct executions with slowdown are exponentially penalized; and
(iii) failed executions incur a fixed penalty.
\end{definition}

We further define the \emph{Geometric Mean Rectified Speedup (GMRS)} as:
\begin{equation}
GMRS_t = \left(\prod_{i=1}^{N} \tilde{s}_{t,i}\right)^{1/N}
\label{eq:gmrs}
\end{equation}
where \(N\) is the total number of evaluated test samples.

\begin{proposition}
\label{prop:equivalence-St-GMRS}
The macro-level Speedup Score $S_t$ defined in Eq.~\ref{eq:St-macro-form} is equivalent to the geometric mean of per-sample rectified speedups:
\[
S_t = \alpha^{\lambda}\cdot \beta^{\lambda \eta p} \cdot b^{1-\lambda} = GMRS_t
\]
\end{proposition}

\begin{proof}
Consider a benchmark containing \(N\) test samples in total.
Among them, let \(M\) be the number of correctly executed samples, and \(K\) be the subset of those whose raw speedup is less than one (i.e., slowdowns).
Formally, we define:
\[
\begin{aligned}
M &= \#\{i \mid \mathrm{correct}_{t,i}\}, &
\lambda &= \tfrac{M}{N}, \\
K &= \#\{i \mid \mathrm{correct}_{t,i},\, s_i < 1\}, &
\eta &= \tfrac{K}{M}.
\end{aligned}
\]
Thus, the sample space can be partitioned as:
\begin{itemize}
    \item $(M-K)$ correctly executed and accelerated samples (\(s_i \ge 1\));
    \item $K$ correct executions with slowdown (\(s_i < 1\));
    \item $(N-M)$ failed or incorrect samples.
\end{itemize}

Expanding Eq.~\eqref{eq:gmrs} under this partition gives:
\begin{align}
GMRS_t^N
&= \prod_{i: \mathrm{correct},\, s_i \ge 1} s_i
   \cdot \prod_{i: \mathrm{correct},\, s_i < 1} s_i^{p+1}
   \cdot \prod_{i: \neg \mathrm{correct}} b
\label{eq:prod_decomp}
\end{align}

Define the geometric means:
\[
\begin{alignedat}{2}
\alpha &= \left(\prod_{i:\mathrm{correct}} s_i\right)^{1/M}, &\quad
\beta &= \left(\prod_{i:\mathrm{correct},\, s_i < 1} s_i\right)^{1/K}
\end{alignedat}
\]
Hence,
\[
\prod_{i:\mathrm{correct}} s_i = \alpha^M, \quad
\prod_{i:\mathrm{correct},\, s_i < 1} s_i = \beta^K,
\]
and consequently:
\[
\prod_{i:\mathrm{correct},\, s_i \ge 1} s_i = \frac{\alpha^M}{\beta^K}
\]

Substituting into Eq.~\eqref{eq:prod_decomp}:
\begin{align*}
GMRS_t^N
&= \frac{\alpha^M}{\beta^K} \cdot (\beta^K)^{p+1} \cdot b^{N-M} \nonumber\\
&= \alpha^M \cdot \beta^{Kp}\cdot b^{N-M}
\end{align*}
Taking the $N$-th root gives:
\begin{equation*}
GMRS_t = \alpha^{M/N}\cdot \beta^{Kp/N} \cdot b^{(N-M)/N}
\end{equation*}
Substituting $\lambda = M/N$ and $\eta = K/M$ yields:
\begin{equation}
GMRS_t = \alpha^{\lambda}\cdot \beta^{\lambda \eta p} \cdot b^{1-\lambda}
\end{equation}
This matches exactly the macro-level formulation in Eq.~\ref{eq:St-macro-form}, completing the proof.
\end{proof}

\newpage
\section{Sample-level Interpretation of \texorpdfstring{$ES_t$}{ESt}}
\label{sec:appendix-analysis-ESt}

This section provides the sample-level interpretation of the Error-aware Speedup Score $ES_t$.
We show that the aggregated error penalty $\gamma_t$ in Eq.~(\ref{eq:ESt-macro-form})
is equivalent to the geometric mean of per-sample penalty factors,
and that $ES_t$ can be viewed as the geometric mean of per-sample \emph{error-aware rectified speedup}.

\medskip
\begin{definition}
For each erroneous sample $i$, let $c_i\!\in\!\{1,2,3\}$ denote its error code
(\(1\) for accuracy errors, \(2\) for execution crashes, \(3\) for compilation failures).
Given tolerance level $t$, we define the sample-level \emph{penalty factor} as
\begin{equation}
r_{t,i} =
\begin{cases}
b, & t < c_i,\\[3pt]
1, & \text{otherwise},
\end{cases}
\end{equation}
where $b\in(0,1)$ is the base penalty introduced in Section~\ref{subsubsec:metrics-benchmark}.
This factor equals $b$ if the current tolerance level $t$ does not forgive the error type $c_i$,
and $1$ otherwise.
\end{definition}

\smallskip
\begin{proposition}
The aggregated penalty $\gamma_t$ can be written as the geometric mean of $\{r_{t,i}\}$ over all erroneous samples:
\begin{equation}
\gamma_t = \left(\prod_{i=1}^{E} r_{t,i}\right)^{1/E}
\end{equation}
where $E$ is the total number of erroneous samples.
\end{proposition}
\smallskip
\noindent\emph{Proof.}
From the definition of $r_{t,i}$,
each term in the product $\prod_i r_{t,i}$ contributes $b$ iff $t<c_i$:
\[
\prod_{i=1}^{E} r_{t,i}
 = \prod_{i=1}^{E} b^{\mathbb{1}(t<c_i)}
 = b^{\sum_{i=1}^{E}\mathbb{1}(t<c_i)}
\]
Grouping by error code $c$ and letting $E_c$ denote the number of samples with code $c$, we have
\[
\sum_{i=1}^{E}\mathbb{1}(t<c_i)
   = \sum_{c=1}^{3}E_c\,\mathbb{1}(t<c)
\]
thus
\[
\prod_{i=1}^{E} r_{t,i}
   = b^{\sum_{c=1}^{3}E_c\,\mathbb{1}(t<c)}
\]
Taking the $E$-th root gives
\[
\left(\prod_{i=1}^{E} r_{t,i}\right)^{1/E}
   = b^{\frac{1}{E}\sum_{c=1}^{3}E_c\,\mathbb{1}(t<c)}
   = b^{\sum_{c=1}^{3}\pi_c\,\mathbb{1}(t<c)}
\]
where $\pi_c=E_c/E$ is the fraction of error code $c$ among all erroneous samples.
The right-hand side is exactly the definition of $\gamma_t$ in Eq.~(\ref{eq:gamma_def}).
\hfill$\square$

\medskip
\begin{definition}(Error-aware rectified speedup).
Extending the rectified speedup defined in Appendix~\ref{sec:appendix-analysis-st},
we introduce the per-sample \emph{error-aware rectified speedup}:
\begin{equation}
\text{error\_aware\_rectified\_speedup}_{t,i} =
\begin{cases}
\mathrm{speedup}_i, & \mathrm{correct}_{t,i}\land \mathrm{speedup}_i\!\ge1,\\[3pt]
\mathrm{speedup}_i^{p+1}, & \mathrm{correct}_{t,i}\land \mathrm{speedup}_i\!<1,\\[3pt]
r_{t,i}, & \text{otherwise.}
\end{cases}
\end{equation}
\end{definition}
\begin{proposition}
Since $\gamma_t$ acts as the geometric mean of $\{r_{t,i}\}$ for a given $t$,
the macro-level metric $ES_t$ in Eq.~(\ref{eq:ESt-macro-form}) can be trivially proved as the geometric mean of
these error-aware rectified speedup values:
\begin{equation}
ES_t = \left(\prod_{i=1}^{N}\text{error\_aware\_rectified\_speedup}_{t,i}\right)^{1/N}
\end{equation}
\end{proposition}

\newpage
\section{Details of Benchmark Experiments}
\label{sec:benchmark-details}
\begin{figure}[h]
    \centering
    \captionsetup{font=small}
    \includegraphics[width=0.7\linewidth]{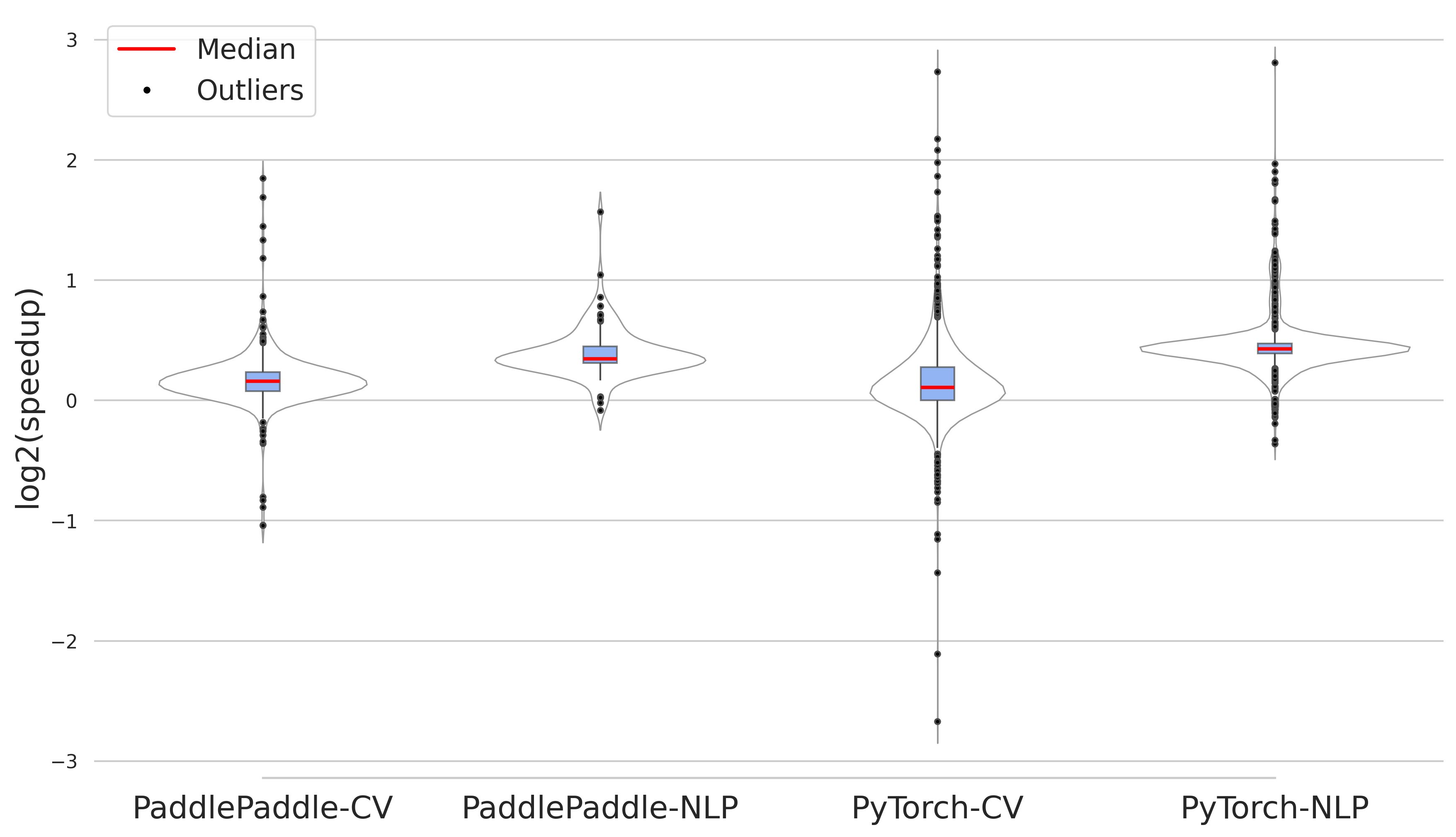}
    \caption{\textbf{Violin Plots of Per-sample Speedups (log$_2$(speedup)).} It shows distribution of per-sample speedups under correctness checks $t=0$.
Each violin shows the speedup spread for a compiler-task pair. Wider regions denote more frequent speedup values. 
Samples with log$_2$(speedup) <0 indicate performance degradation. }
    \label{fig:experiment-violinplot}
\end{figure}

Experiments were conducted on a server with dual Intel Xeon Platinum 8563C CPUs, 2 TB DRAM, and eight NVIDIA H20-3e GPUs (144 GB each). Only one GPU was used to ensure a single-device setting. The software environment included Python 3.11, PyTorch 2.8, and PaddlePaddle 3.2.

PaddlePaddle and PyTorch show broadly consistent dataset distributions. For graphs with fewer than 256 operators, the proportions are 6\% vs.\ 8.5\% in NLP (PaddlePaddle vs. PyTorch) and 13\% vs.\ 5.6\% in CV. For graphs with more than 256 operators, the proportions are 94\% vs.\ 91.5\% in NLP and 87\% vs.\ 94.4\% in CV. These residual differences will be resolved through a high-level IR transformation tool.


The four tables below (Tables~\ref{tab:paddle-nlp-values}-\ref{tab:pytorch-cv-values}) report detailed scores for the NLP / CV models tested on PaddlePaddle and PyTorch, corresponding to the statistical results behind each point of the $S_t$ plots in Section~3.2.1 and the $ES_t$ plots in Section~3.2.2.  Since $S_t$ does not consider the region $t>0$, the corresponding entries are marked ``-'' to indicate ``empty''.  $\gamma$ is used only when computing $ES_t$.  For points that contain no correct samples, $\eta$ is initialized to 0. $\alpha$ and $\beta$ are initialized to 1 if no correct samples exist.

\begin{table}[h!]
\centering

\begin{minipage}[t]{0.49\textwidth}
\centering
\resizebox{\linewidth}{!}{%
\begin{tabular}{|c||ccccccc|}
\hline
\textbf{t} & \textbf{$\alpha$} & \textbf{$\beta$} & \textbf{$\lambda$} & \textbf{$\eta$} & \textbf{S(t)} & \textbf{$\gamma$} & \textbf{ES(t)} \\ \hline \hline
-10	&	1.000 	&	1.000 	&	0.000 	&	0.000 	&	0.100 	&	0.100 	&	0.100 	\\
-9	&	1.000 	&	1.000 	&	0.000 	&	0.000 	&	0.100 	&	0.100 	&	0.100 	\\
-8	&	1.000 	&	1.000 	&	0.000 	&	0.000 	&	0.100 	&	0.100 	&	0.100 	\\
-7	&	1.441 	&	1.000 	&	0.009 	&	0.000 	&	0.103 	&	0.100 	&	0.103 	\\
-6	&	1.321 	&	0.858 	&	0.575 	&	0.049 	&	0.442 	&	0.100 	&	0.442 	\\
-5	&	1.284 	&	0.878 	&	0.953 	&	0.050 	&	1.139 	&	0.100 	&	1.139 	\\
-4	&	1.282 	&	0.878 	&	0.962 	&	0.049 	&	1.165 	&	0.100 	&	1.165 	\\
-3	&	1.278 	&	0.878 	&	0.991 	&	0.048 	&	1.248 	&	0.100 	&	1.248 	\\
-2	&	1.278 	&	0.878 	&	0.991 	&	0.048 	&	1.248 	&	0.100 	&	1.248 	\\
-1	&	1.278 	&	0.878 	&	0.991 	&	0.048 	&	1.248 	&	0.100 	&	1.248 	\\
0	&	1.278 	&	0.878 	&	0.991 	&	0.048 	&	1.248 	&	0.100 	&	1.248 	\\
1	&	1.278 	&	0.878 	&	0.991 	&	0.048 	&	-	&	0.100 	&	1.248 	\\
2	&	1.278 	&	0.878 	&	0.991 	&	0.048 	&	-	&	0.100 	&	1.248 	\\
3	&	1.278 	&	0.878 	&	0.991 	&	0.048 	&	-	&	1.000 	&	1.276 	\\
4	&	1.278 	&	0.878 	&	0.991 	&	0.048 	&	-	&	1.000 	&	1.276 	\\
\hline
\end{tabular}%
}
\caption{Values for PaddlePaddle-NLP}
\label{tab:paddle-nlp-values}
\end{minipage}
\hfill
\begin{minipage}[t]{0.49\textwidth}
\centering
\resizebox{\linewidth}{!}{%
\begin{tabular}{|c||ccccccc|}
\hline
\textbf{t} & \textbf{$\alpha$} & \textbf{$\beta$} & \textbf{$\lambda$} & \textbf{$\eta$} & \textbf{S(t)} & \textbf{$\gamma$} & \textbf{ES(t)} \\ \hline \hline
-10	&	0.981 	&	0.927 	&	0.010 	&	0.800 	&	0.102 	&	0.100 	&	0.102 	\\
-9	&	1.038 	&	0.927 	&	0.013 	&	0.615 	&	0.103 	&	0.100 	&	0.103 	\\
-8	&	1.170 	&	0.840 	&	0.029 	&	0.321 	&	0.107 	&	0.100 	&	0.107 	\\
-7	&	1.175 	&	0.840 	&	0.030 	&	0.310 	&	0.108 	&	0.100 	&	0.108 	\\
-6	&	1.344 	&	0.822 	&	0.423 	&	0.064 	&	0.300 	&	0.100 	&	0.300 	\\
-5	&	1.342 	&	0.817 	&	0.930 	&	0.046 	&	1.118 	&	0.100 	&	1.118 	\\
-4	&	1.345 	&	0.817 	&	0.933 	&	0.046 	&	1.129 	&	0.100 	&	1.129 	\\
-3	&	1.348 	&	0.817 	&	0.941 	&	0.045 	&	1.156 	&	0.100 	&	1.156 	\\
-2	&	1.354 	&	0.818 	&	0.963 	&	0.045 	&	1.229 	&	0.100 	&	1.229 	\\
-1	&	1.360 	&	0.818 	&	0.975 	&	0.045 	&	1.275 	&	0.100 	&	1.275 	\\
0	&	1.363 	&	0.818 	&	0.977 	&	0.044 	&	1.284 	&	0.100 	&	1.284 	\\
1	&	1.363 	&	0.818 	&	0.977 	&	0.044 	&	- 	&	0.100 	&	1.284 	\\
2	&	1.363 	&	0.818 	&	0.977 	&	0.044 	&	- 	&	0.100 	&	1.284 	\\
3	&	1.363 	&	0.818 	&	0.977 	&	0.044 	&	- 	&	1.000 	&	1.353 	\\
4	&	1.363 	&	0.818 	&	0.977 	&	0.044 	&	- 	&	1.000 	&	1.353 	\\
\hline
\end{tabular}%
}
\caption{Values for PyTorch-NLP}
\label{tab:pytorch-nlp-values}
\end{minipage}
\end{table}
\vspace{1cm}

\begin{table}[H]
\centering
\begin{minipage}[t]{0.49\textwidth}
\centering
\resizebox{\linewidth}{!}{%
\begin{tabular}{|c||ccccccc|}
\hline
\textbf{t} & \textbf{$\alpha$} & \textbf{$\beta$} & \textbf{$\lambda$} & \textbf{$\eta$} & \textbf{S(t)} & \textbf{$\gamma$} & \textbf{ES(t)} \\ \hline \hline
-10	&	1.056 	&	0.917 	&	0.054 	&	0.739 	&	0.114 	&	0.100 	&	0.114 	\\
-9	&	1.056 	&	0.917 	&	0.054 	&	0.739 	&	0.114 	&	0.100 	&	0.114 	\\
-8	&	1.056 	&	0.917 	&	0.054 	&	0.739 	&	0.114 	&	0.100 	&	0.114 	\\
-7	&	1.056 	&	0.920 	&	0.080 	&	0.529 	&	0.121 	&	0.100 	&	0.121 	\\
-6	&	1.117 	&	0.906 	&	0.228 	&	0.216 	&	0.173 	&	0.100 	&	0.173 	\\
-5	&	1.137 	&	0.906 	&	0.526 	&	0.129 	&	0.359 	&	0.100 	&	0.359 	\\
-4	&	1.131 	&	0.915 	&	0.732 	&	0.109 	&	0.591 	&	0.100 	&	0.591 	\\
-3	&	1.125 	&	0.912 	&	0.892 	&	0.105 	&	0.866 	&	0.100 	&	0.866 	\\
-2	&	1.123 	&	0.911 	&	0.939 	&	0.105 	&	0.969 	&	0.100 	&	0.969 	\\
-1	&	1.122 	&	0.911 	&	0.965 	&	0.105 	&	1.031 	&	0.100 	&	1.031 	\\
0	&	1.121 	&	0.910 	&	0.981 	&	0.105 	&	1.072 	&	0.100 	&	1.072 	\\
1	&	1.122 	&	0.910 	&	0.993 	&	0.104 	&	- 	&	1.000 	&	1.121 	\\
2	&	1.122 	&	0.910 	&	0.993 	&	0.104 	&	- 	&	1.000 	&	1.123 	\\
3	&	1.122 	&	0.910 	&	0.993 	&	0.104 	&	- 	&	1.000 	&	1.123 	\\
4	&	1.122 	&	0.910 	&	0.993 	&	0.104 	&	- 	&	1.000 	&	1.123 	\\
\hline

\end{tabular}%
}
\caption{Values for PaddlePaddle-CV}
\label{tab:paddle-cv-values}
\end{minipage}
\hfill
\begin{minipage}[t]{0.49\textwidth}
\centering
\resizebox{\linewidth}{!}{%
\begin{tabular}{|c||ccccccc|}
\hline
\textbf{t} & \textbf{$\alpha$} & \textbf{$\beta$} & \textbf{$\lambda$} & \textbf{$\eta$} & \textbf{S(t)} & \textbf{$\gamma$} & \textbf{ES(t)} \\ \hline \hline
-10	&	0.949 	&	0.825 	&	0.345 	&	0.618 	&	0.217 	&	0.100 	&	0.217 	\\
-9	&	0.954 	&	0.829 	&	0.404 	&	0.623 	&	0.249 	&	0.100 	&	0.249 	\\
-8	&	0.957 	&	0.821 	&	0.447 	&	0.609 	&	0.274 	&	0.100 	&	0.274 	\\
-7	&	0.984 	&	0.819 	&	0.605 	&	0.552 	&	0.399 	&	0.100 	&	0.399 	\\
-6	&	1.006 	&	0.814 	&	0.766 	&	0.505 	&	0.586 	&	0.100 	&	0.586 	\\
-5	&	1.030 	&	0.815 	&	0.882 	&	0.472 	&	0.783 	&	0.100 	&	0.783 	\\
-4	&	1.028 	&	0.815 	&	0.886 	&	0.474 	&	0.788 	&	0.100 	&	0.788 	\\
-3	&	1.028 	&	0.815 	&	0.886 	&	0.474 	&	0.788 	&	0.100 	&	0.788 	\\
-2	&	1.028 	&	0.815 	&	0.887 	&	0.473 	&	0.790 	&	0.100 	&	0.790 	\\
-1	&	1.028 	&	0.815 	&	0.887 	&	0.473 	&	0.790 	&	0.100 	&	0.790 	\\
0	&	1.028 	&	0.815 	&	0.887 	&	0.473 	&	0.790 	&	0.100 	&	0.790 	\\
1	&	1.028 	&	0.815 	&	0.887 	&	0.473 	&	- 	&	0.100 	&	0.790 	\\
2	&	1.028 	&	0.815 	&	0.887 	&	0.473 	&	- 	&	0.100 	&	0.790 	\\
3	&	1.028 	&	0.815 	&	0.887 	&	0.473 	&	- 	&	1.000 	&	1.025 	\\
4	&	1.028 	&	0.815 	&	0.887 	&	0.473 	&	- 	&	1.000 	&	1.025 	\\
\hline
\end{tabular}%
}
\caption{Values for PyTorch-CV}
\label{tab:pytorch-cv-values}
\end{minipage}
\end{table}

\end{document}